\documentclass[letterpaper, 10 pt, conference]{ieeeconf}
\IEEEoverridecommandlockouts

\usepackage{cite}
\usepackage{amsmath,amssymb,amsfonts}
\usepackage[ruled,vlined,linesnumbered,noend]{algorithm2e}
\usepackage{graphicx}
\usepackage{textcomp}
\usepackage{xcolor}
\usepackage{ragged2e}
\usepackage{optidef}
\usepackage{subcaption}      % use this OR subfig, never both -> keeping subcaption
\usepackage{cleveref}
\usepackage{multirow}
\usepackage{url}
\usepackage{diagbox}
\usepackage[font=footnotesize]{caption}
\usepackage{soul}
\usepackage{float}
\usepackage{resizegather}
\usepackage{makecell}
\usepackage{booktabs}
\usepackage{tabularx}
\let\labelindent\relax
\usepackage{enumitem}
\usepackage{import}
\newtheorem{theorem}{Theorem}

\newtheorem{assumption}{\bf{Assumption}}

\usepackage{tikz}
\usepackage{tkz-euclide}
\usetikzlibrary{shapes.geometric}

\def\BibTeX{{\rm B\kern-.05em{\sc i\kern-.025em b}\kern-.08em
    T\kern-.1667em\lower.7ex\hbox{E}\kern-.125emX}}

\begin{document}

\title{\textbf{Robust and Scalable UAV-UGV Planning for Aerial Monitoring under Energy and Precedence Constraints}}

\author{Cahit Ikbal Er and Yasin Yaz{\i}c{\i}o\u{g}lu%
\thanks{C. I. Er is with the Department of Mechanical and Industrial Engineering,
Northeastern University, Boston, MA, USA (email: er.c@northeastern.edu).}
\thanks{Y. Yaz{\i}c{\i}o\u{g}lu is with the Departments of Mechanical and Industrial
Engineering and Electrical and Computer Engineering, Northeastern University,
Boston, MA, USA (email: y.yazicioglu@northeastern.edu).}
}

\maketitle
\pagestyle{empty}

%%%%%%%%%%%%%%%%%%%%%%%%%%%%%%%%%%%%%%%%%%%%%%%%%%%%%%%%%%%%%%%%%%%%%%%%%%%%%%

\begin{abstract}
We consider planning for unmanned aerial vehicle (UAV) and unmanned ground vehicle (UGV) teams in aerial monitoring missions with ordering requirements on the monitoring points to be visited by the UAV.  The system consists of multiple UAV-UGV teams, each comprising one UAV and one UGV. UAV energy limitations are modeled through endurance, defined as the maximum flight time on a full battery. The UGVs serve as mobile charging stations and transport the UAVs between consecutive aerial tours. Given a set of monitoring points subject to ordering constraints and the teams’ initial and terminal positions, the objective is to minimize mission completion time while satisfying UAV energy limits and the required visit order under bounded travel-time disturbances that may arise during execution. Unlike prior work, the plan jointly determines vehicle routes and their timing, allowing teams to slow down as needed to satisfy the ordering constraints. We develop a polynomial-time algorithm that generates efficient solutions to the underlying NP-hard planning problem. We establish feasibility guarantees, characterize the algorithm’s computational complexity, and evaluate its performance via numerical comparisons.
\end{abstract}

\section{Introduction}
Unmanned aerial vehicles (UAVs) are used in various domains such as agriculture and environmental monitoring (e.g., \cite{tokekar2016sensor, ren2019uavmonitoring}). However, their limited energy restricts missions defined over large areas or long time-horizons. Using unmanned ground vehicles (UGVs) as mobile charging stations addresses this limitation by providing a flexible alternative to fixed recharging infrastructure.
Recently, energy-constrained planning of UAV-UGV teams has gained significant attention (e.g., \cite{yu2018algorithms,ropero2019terra,ramasamy2021cooperative,chen2021delivery, er2026energy, er2026rspect, fowler2026prospect}).  However, these methods generally assume that points of interest may be visited in any order. 

In practice, aerial monitoring missions may require some locations to be visited before others. In multi-robot task allocation, such ordering requirements are formalized as cross-schedule dependencies \cite{korsah2013comprehensive}. Exact methods have addressed vehicle-routing problems with temporal synchronization \cite{bredstrom2007branch}, while polynomial-time sequencers handle task allocation under temporal and spatial constraints \cite{gombolay2018fast}; however, these approaches generally do not model vehicle energy. Similarly, \cite{bischoff2020multi} schedules cooperative tasks using greedy construction and local search, but focuses on task scheduling rather than energy-constrained spatial planning. More recently, \cite{sojo2024optimal} computes optimal precedence-constrained plans for heterogeneous UAV teams under energy limits, though its computational complexity grows significantly with team size and recharging relies on static infrastructure.

In this paper, we propose a heuristic algorithm for robust UAV-UGV planning in aerial monitoring missions with precedence constraints. The goal is to minimize the overall mission completion time while respecting the energy constraints and the required visit order under bounded travel-time disturbances. We formulate this planning problem as a mixed-integer program, which is NP-hard. Our main contribution is an efficient heuristic (Alg. \ref{alg:precept}) that jointly determines vehicle routes and their timing, with formal guarantees on feasibility and worst-case time complexity (Theorem \ref{thm:main}). We validate our approach via numerical comparisons against branch-and-cut, simulated annealing, and a heuristic baseline that assigns all locations that are weakly connected by precedence constraints to the same team.

\color{black}

\section{Problem Formulation}
\label{sec:problem_formulation}
We consider an aerial monitoring mission executed by multiple UAV-UGV teams. The energy-constrained UAVs must visit a set of monitoring points, supported by the UGVs which serve as mobile charging stations and transport the UAVs between take-off and landing locations. Two features distinguish this formulation from prior work (e.g., \cite{er2026rspect, fowler2026prospect}). First, precedence requirements constrain the order in which the monitoring points may be visited. Second, the plan selects not only where the vehicles go but how long each leg takes, allowing a team to slow down if needed rather than reroute to satisfy these ordering constraints.

Uncertainty is treated in the deterministic, worst-case sense of robust optimization under interval uncertainty \cite{ben2009robust}. Following \cite{er2026rspect}, the plan incorporates margins $\delta_a$ and $\delta_g$ to ensure that unexpected travel delays do not deplete a UAV's energy. However, these margins do not preserve the mission's visit order: a delay at one point could cause it to be visited after a point it must precede. To prevent this, we introduce an ordering margin, $\delta_o$. This enforces a minimum temporal separation between precedence-constrained points, ensuring bounded delays do not violate the required visit order.

In the remainder, $\mathbb{R}$ and $\mathbb{R}_{\geq 0}$ denote the sets of real and non-negative real numbers. Let the environment be
\begin{equation}
\label{eq:env}
    \mathcal{Q} \coloneqq \{(x,y,z) \in \mathbb{R}_{\geq 0}^{3} \mid x \leq \bar{x},\, y \leq \bar{y},\, z \leq \bar{z}\},
\end{equation}
containing a feasible space $\mathcal{Q}_f \subseteq \mathcal{Q}$ determined by the obstacles known a priori, with ground and air subsets $\mathcal{Q}^g_f {=} \{(x,y,z) \in \mathcal{Q}_f \mid z {=} 0\}$ and $\mathcal{Q}^a_f {=} \mathcal{Q}_f \setminus \mathcal{Q}^g_f$. Both $\mathcal{Q}_f$ and $\mathcal{Q}^g_f$ are assumed connected, so that a feasible path exists between any pair of points within the respective sets, and $\ell : \mathcal{Q}_f \times \mathcal{Q}_f \rightarrow \mathbb{R}_{\geq 0}$ returns the length of the shortest path between two points of $\mathcal{Q}_f$. Additionally, let $\pi: \mathcal{Q}_f \to \mathcal{Q}_f^{g}$ be the feasible ground projection function, defined as $\pi(p) = \arg\min_{q \in Q_f^g} \ell(p,q)$.

The mission is specified by a set of $n$ aerial monitoring points $\mathcal{P}_{\text{UAV}} = \{p_1,\dots,p_n\} \subset \mathcal{Q}^a_f$. There are $m$ teams (UAV-UGV pairs) indexed by $\mu \in \{1,\dots,m\}$, each starting from a start position $p^\mu_{\text{s}} \in \mathcal{Q}^g_f$ and required to reach a specified final position $p^\mu_{\text{f}} \in \mathcal{Q}^g_f$. The teams must collectively visit every point of $\mathcal{P}_{\text{UAV}}$\footnotemark[1] in an order consistent with a given precedence relation, specified as a directed acyclic graph (DAG) $G_\prec {=} (\mathcal{P}_{\text{UAV}}, E_\prec)$ whose edges $(p,q) \in E_\prec$ require $p$ to be visited before $q$.\footnote{No assumption is made on how $\mathcal{P}_{\text{UAV}}$ and $E_\prec$ arise: the points may encode imaging viewpoints at a desired resolution, and the ordering a requirement such as surveying the perimeter of a site before entering it.} Acyclicity ensures that a valid visit order always exists, as any topological sort of $G_\prec$ provides one \cite{cormen2009algorithms}. The graph $G_\prec$ can encode any partial order and need not be connected; isolated vertices carry no ordering requirements.

A UAV cannot fly continuously longer than its maximum flight time $\overline{\tau}_a$. Because power consumption is approximately constant for UAVs flown at constant average speed (e.g., \cite{abdilla2015power}), flight time acts as a direct proxy for energy consumption, making $\overline{\tau}_a$ the energy budget, as is standard in the literature (e.g., \cite{yu2018algorithms, ropero2019terra, chen2021delivery, ramasamy2021cooperative, er2026rspect,  er2026energy, fowler2026prospect}). UGVs carry larger batteries and are assumed to have sufficient energy for the mission.

To respect $\overline{\tau}_a$, the UGVs transport the UAVs through a sequence of \emph{tours}. In each tour, a UAV takes off from a \emph{release} point on the ground, visits a subset of $\mathcal{P}_{\text{UAV}}$, and lands on the UGV at the corresponding \emph{collect} point to recharge. The UGV then carries it to the next release point, or to $p^\mu_{\text{f}}$ if the mission is complete. Because a single tour can visit at most $n$ monitoring points, the plan of team $\mu$ is encoded as the matrix
\begin{equation*}
    X^{\mu} = \begin{bmatrix}
        X^{\mu}_{1,1} & \cdots & X^{\mu}_{1,n+2} \\
        \vdots & \ddots & \vdots \\
        X^{\mu}_{n,1} & \cdots & X^{\mu}_{n,n+2}
    \end{bmatrix},
\end{equation*}
where $X^\mu_{i,j} \in \mathcal{Q}_f$ is a release point if $j{=}1$, a collect point if $j{=}n{+}2$, and a monitoring point to be visited by the UAV if $j \in \{2,\dots,n{+}1\}$. Row $i$, denoted $X^\mu_i$, represents one such tour. Setting $X^\mu_{i,1} {=} \cdots {=} X^\mu_{i,n+2}$ in a row yields a trivial tour of zero length, ensuring plans with fewer than $n$ tours are representable. The overall mission plan is $X {=} \{X^1,\dots,X^m\}$.

Let $\underline{\theta}_a : \mathcal{Q}_f \times \mathcal{Q}_f \rightarrow \mathbb{R}_{\geq 0}$ and $\underline{\theta}_g : \mathcal{Q}_f^g \times \mathcal{Q}_f^g \rightarrow \mathbb{R}_{\geq 0}$ be the fastest travel times in which the UAVs and the UGVs can traverse a shortest path between respective points,\footnote{{We do not commit to specific forms for $\underline{\theta}_a$, $\underline{\theta}_g$ and $\tau_c$. The simplest instance takes shortest paths and divides them with fixed speeds, with $\tau_c$ modeling either proportional recharging or a constant battery swap.}} with $\underline{\theta}_a(p,p)=\underline{\theta}_g(p,p)=0$. The duration of each leg is assigned independently and lower-bounded by these fastest travel times. These durations form the \emph{timing} of the mission. Let the timing for team $\mu$ be defined by the tuple $\Theta^\mu = (\Theta_a^\mu, \Theta_g^\mu, \Theta_{ga}^\mu, \theta_{\text{s}}^\mu, \theta_{\text{f}}^\mu)$. Here, $\Theta_a^\mu \in \mathbb{R}_{\geq 0}^{n \times (n+1)}$ is the flight timing matrix, where entry $\theta_a^{\mu, i, j}$ is the UAV's flight time from $X_{i,j}^\mu$ to $X_{i, j+1}^\mu$. The vector $\Theta_g^\mu \in \mathbb{R}_{\geq 0}^n$ defines UGV relocation times, with entry $\theta_g^{\mu, i}$ denoting the time from the release $X_{i,1}^\mu$ to the collect $X_{i, n+2}^\mu$ within tour $i$. Between consecutive tours, the gap timing vector $\Theta_{ga}^\mu \in \mathbb{R}_{\geq 0}^{n-1}$ provides the transport time $\theta_{ga}^{\mu, i}$ from $X_{i, n+2}^\mu$ to $X_{i+1, 1}^\mu$. Finally, the scalars $\theta_{\text{s}}^\mu$ and $\theta_{\text{f}}^\mu$ denote the initial and final transport times from $p_{\text{s}}^\mu$ to $X_{1,1}^\mu$ and from $X_{n, n+2}^\mu$ to $p_{\text{f}}^\mu$, respectively. The overall mission timing is $\Theta = \{\Theta^\mu\}_{\mu=1}^m$.
% \emph{pacing}
The times to complete the UAV and UGV portions of tour $i$ are $\tau_a(X^\mu_i, \Theta^{\mu}) = \sum_{j=1}^{n+1} \theta_a^{\mu,i,j}$ and $\tau_g(X^\mu_i, \Theta^{\mu}) = \theta_g^{\mu,i}$, while the subsequent UGV transport time is $\tau_{ga}(X^\mu_i, \Theta^{\mu}) = \theta_{ga}^{\mu,i}$. Concurrently, the UAV requires a recharge time defined by $\tau_c\big(\max(\tau_a(X^\mu_i, \Theta^{\mu}), \tau_g(X^\mu_i, \Theta^{\mu}))\big) \geq 0$, where $\tau_c : \mathbb{R}_{\geq 0} \to \mathbb{R}_{\geq 0}$ with $\tau_c(0)=0$ maps the airborne time on a tour to the time needed to recharge the UAV.

The total execution time for team $\mu$, denoted $\tau(X^{\mu}, \Theta^{\mu})$, is the sum of its initial and final transport times, the time spent executing the tours, and the time spent between tours while the UGV carries and recharges the UAV:
\begin{equation}
\begin{aligned}
\tau(X^{\mu}, \Theta^{\mu}) {=}\ & \theta_{\text{s}}^{\mu} + \theta_{\text{f}}^{\mu} + \sum_{i=1}^{n} \max(\tau_a(X^{\mu}_{i}, \Theta^{\mu}),\, \tau_g(X^{\mu}_{i}, \Theta^{\mu})) \\[-2pt]
& + \sum_{i=1}^{n-1} \max\Big(\tau_{ga}(X^{\mu}_{i}, \Theta^{\mu}),\,
    \tau_c\big(\max(\tau_a(X^{\mu}_{i}, \Theta^{\mu}),\,
    \tau_g(X^{\mu}_{i}, \Theta^{\mu}))\big)\Big).
\end{aligned}
\label{eq:objective_function}
\end{equation}

A \textit{plan} $X$ and \textit{timing} $\Theta$ determine the time each monitoring point is reached. Let $t_p(X, \Theta)$ denote the \emph{visit time} of $p$, defined as the earliest time any UAV reaches $p$ relative to the mission start. Since $t_p(X, \Theta)$ accumulates preceding leg durations, a delay in any leg propagates to all subsequent visits dependent on it.

We define the planning problem as minimizing the makespan over the $m$ teams as:
{\small
\begin{mini!}|s|[2]
    {X,\, \Theta}
    {\displaystyle \max_{\mu \in \{1,\dots,m\}} \; \tau(X^{\mu}, \Theta^{\mu}) \label{eq:cost}}
    {\label{opt_problem}}
    {}
    \addConstraint{\forall p_k \in \mathcal{P}_{\text{UAV}},\; \exists\, X^{\mu}_{i,j} = p_k \label{eq:cov}}
    \addConstraint{\tau_a(X^{\mu}_{i}, \Theta^{\mu}) + \delta_a \leq \overline{\tau}_{a},}{\, \scriptstyle \forall i,\forall\mu \label{eq:air}}
    \addConstraint{\tau_g(X^{\mu}_{i}, \Theta^{\mu}) + \delta_g \leq \overline{\tau}_{a},}{\, \scriptstyle \forall i,\forall\mu \label{eq:ground}}
    \addConstraint{X^{\mu}_{i,1},\, X^{\mu}_{i,n+2} \in \mathcal{Q}_f^g,}{\, \scriptstyle \forall i,\forall\mu \label{eq:ground_pts}}
    \addConstraint{X^{\mu}_{i,j} \in \mathcal{P}_{\text{UAV}} \cup \{X^{\mu}_{i,j{-}1}\},}{\, \scriptstyle \forall j{\in}\{2,\dots,n{+}1\},\forall i,\forall\mu \label{eq:struct}}
    \addConstraint{\theta_a^{\mu,i,j} \geq \underline{\theta}_a(X^{\mu}_{i,j}, X^{\mu}_{i,j+1}),}{\, \scriptstyle \forall j{\in}\{1,\dots,n{+}1\},\forall i,\forall\mu \label{eq:air_floor}}
    \addConstraint{\theta_g^{\mu,i} \geq \underline{\theta}_g(X^{\mu}_{i,1}, X^{\mu}_{i,n+2}),}{\, \scriptstyle \forall i,\forall\mu \label{eq:ground_floor}}
    \addConstraint{\theta_{ga}^{\mu,i} \geq \underline{\theta}_g(X^{\mu}_{i,n+2}, X^{\mu}_{i+1,1}),}{\, \scriptstyle \forall i{\in}\{1,\dots,n{-}1\},\forall\mu \label{eq:transport_floor}}
    \addConstraint{\theta_{\text{s}}^{\mu} \geq \underline{\theta}_g(p^\mu_{\text{s}}, X^{\mu}_{1,1}),\ \ \theta_{\text{f}}^{\mu} \geq \underline{\theta}_g(X^{\mu}_{n,n+2}, p^\mu_{\text{f}}),}{\, \scriptstyle \forall\mu \label{eq:boundary_floor}}
    \addConstraint{t_p(X, \Theta) + \delta_o \leq t_q(X, \Theta),}{\, \scriptstyle \forall(p,q){\in}E_\prec \label{eq:prec}.}
\end{mini!}}
\color{black}
Constraint \eqref{eq:cov} ensures that every monitoring point is visited by at least one team. Constraints \eqref{eq:air} and \eqref{eq:ground} encode the UAVs' energy limitation: a UAV arriving at the collect point ahead of the UGV must hover until it arrives, so it remains airborne for the longer of its own flight and the UGV's relocation, both are bounded by $\overline{\tau}_a$ even if the corresponding travel times increase by $\delta_a$ and $\delta_g$ during the mission. Since the plan selects those times, \eqref{eq:air} and \eqref{eq:ground} also limit how much longer a leg may take: a point can be delayed only as long as the UAVs' energy permits. Constraint \eqref{eq:ground_pts} places all release and collect points on feasible ground, and \eqref{eq:struct} ensures that each entry of a row other than the first and the last is either a point of $\mathcal{P}_{\text{UAV}}$ or a repetition of the previous entry, imposing no restriction on how many points a tour may contain. Constraints \eqref{eq:air_floor}-\eqref{eq:boundary_floor} keep every selected time at or above the fastest traverse between the points involved. Finally, \eqref{eq:prec} enforces the ordering requirement: whenever $p$ must precede $q$, the $p$ is visited at least $\delta_o$ earlier than $q$.

The three parameters, $\delta_a$, $\delta_g$, and $\delta_o$, provide robustness to disturbances that increase travel times during execution, such as terrain effects, wind, or avoidance of previously unknown obstacles. The margins $\delta_a$ and $\delta_g$ help ensure that each tour remains feasible under the UAV endurance constraint, while $\delta_o$ provides robustness against violations of the precedence constraints. For each \((p,q)\in E_{\prec}\), precedence is preserved as long as the accumulated delay at \(p\) exceeds that at \(q\) by no more than $\delta_o$. Larger values of these margins restrict the feasible set to plans with greater robustness to travel-time disturbances, which can increase the minimum achievable makespan.

\section{Proposed Solution}
\label{sec:proposed_solution}
Problem \eqref{opt_problem} is a mixed-integer program (MIP) that generalizes both the multiple traveling salesperson problem (mTSP) \cite{bektas2006mtsp} and the sequential ordering problem (SOP) \cite{escudero1988sop}\footnote{The mTSP asks a set of agents to collectively visit every node at minimum cost; the SOP asks for a minimum-cost Hamiltonian path respecting a partial order on the nodes with a single agent.} it is therefore NP-hard, and exact methods such as branch and cut \cite{mitchell2002branch} become intractable for large $|\mathcal{P}_{\text{UAV}}|$. We propose Alg. \ref{alg:precept}, a polynomial-time heuristic, to solve \eqref{opt_problem} and present its formal guarantees in this section.\footnote{These concern feasibility and worst-case time complexity. No approximation ratio or bound on solution quality relative to the optimum is established.}
\color{black}

\subsection{Proposed Robust Planning Algorithm}
\label{sec:proposed_alg}

The constraint \eqref{eq:prec} is what makes \eqref{opt_problem} hard to divide among the teams. One could avoid the difficulty by assigning each weakly connected component of $G_\prec$ to an individual team and ordering that team's visits consistently with $G_\prec$, after which every team could be planned independently. We do not, because the number of components then bounds how many teams can be assigned points: when $G_\prec$ forms a single component, one team receives the entire mission regardless of $m$, and \textcolor{black}{all remaining teams are forced to stay idle.} Instead, Alg.~\ref{alg:precept} distributes ordered points across teams, so two points with precedence constraint may be visited by different teams. 

Without \eqref{eq:prec}, the points can be assigned among the $m$ teams and each team can then be planned independently. \textcolor{black}{Such a decoupled sequential approach (e.g., \cite{er2026rspect}) is inadequate under \eqref{eq:prec} since precedence constraints between points assigned to different teams couple the team plans through their visit times.} In our approach, the assignment, the sequences, the tours, and the timing are constructed simultaneously by placing one point per iteration. Accordingly, the visit time of each point is known when it is added to a team's plan, and each new visit is planned with respect to the already determined visit times of its predecessors. Fig.~\ref{fig:precept_output} illustrates the output of Alg.~\ref{alg:precept}.

\begin{figure}[htpb]
    \centering
    \includegraphics[width=1.0\linewidth]{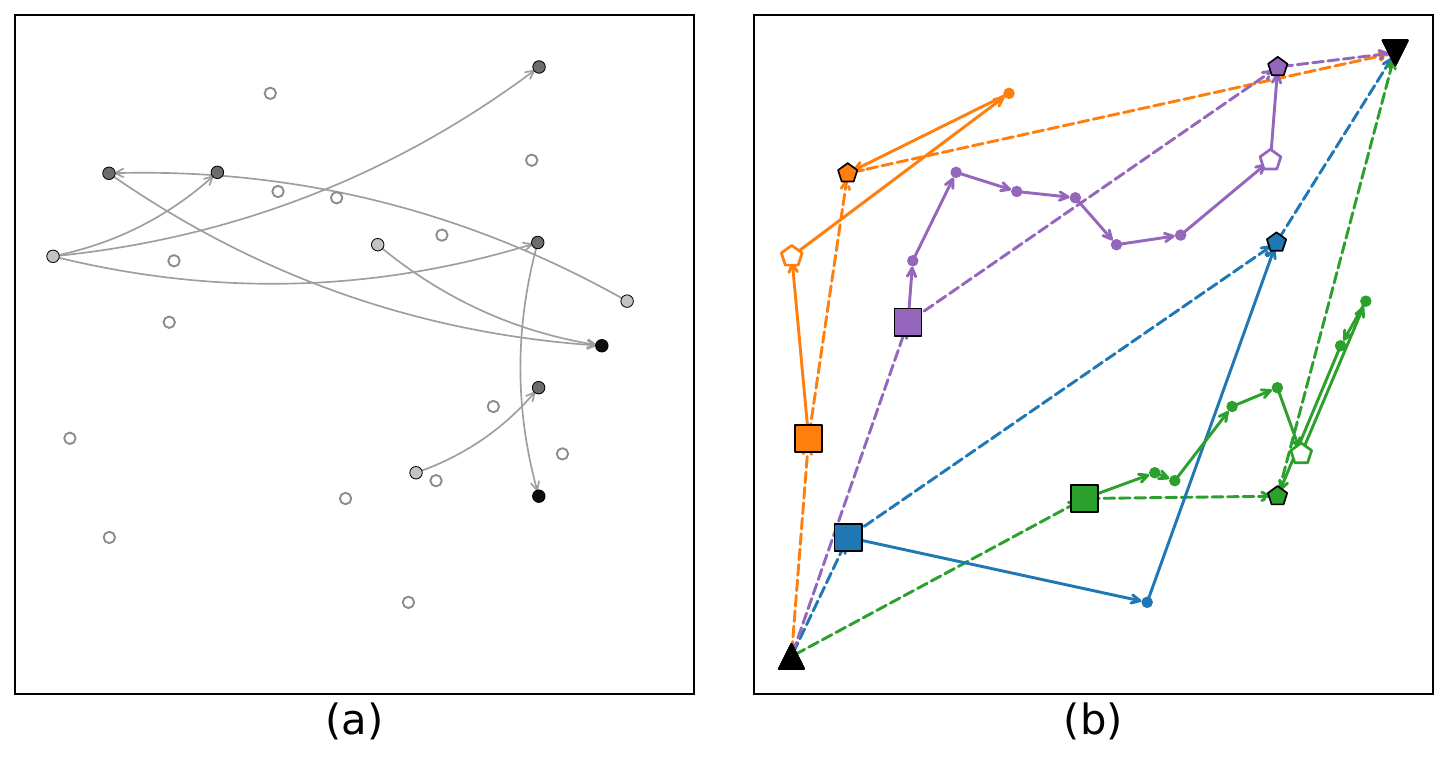}
    \caption{Alg.~\ref{alg:precept} for $n{=}25$ and $m{=}4$. (a)
    $G_\prec {=} (\mathcal{P}_{\text{UAV}}, E_\prec)$, shaded by longest-path depth
    (light for sources, dark for sinks), hollow circles unordered points. (b) The plan
    in space (top view), color-coded by team: squares release points, filled pentagons
    selected collect points, hollow pentagons feasible collect points not taken, solid
    lines UAV flights, dashed lines UGV paths, triangles $p^{\mu}_{\text{s}}$ and
    $p^{\mu}_{\text{f}}$, shared by all teams here. The plan respects the ordering requirements.}
    \label{fig:precept_output}
\end{figure}

\color{blue}
\begin{algorithm}
\renewcommand{\AlCapSty}[1]{\normalfont\footnotesize{\textbf{#1}}\unskip}\scriptsize
\SetAlgoLined
\SetAlgoNoEnd
\DontPrintSemicolon
\SetKwInOut{Input}{Input}
\SetKwInOut{Output}{Output}

\Input{
    $\mathcal{P}_{\text{UAV}}$ ($n$ points to visit), $G_\prec {=} (\mathcal{P}_{\text{UAV}}, E_\prec)$ (Precedence graph),
    $m$ (Number of teams), $\{p^{\mu}_{\text{s}}, p^{\mu}_{\text{f}}\}_{\mu=1}^{m}$ (Team start/finish positions) , $\pi(\cdot)$ (Projection function), $\mathcal{Q}_f^g, \mathcal{Q}_f^a$ (Feasible ground, air space),
    $\underline{\theta}_a, \underline{\theta}_g$ (UAV, UGV travel time lower bounds), $\tau_c$ (Recharge function),
    $\overline{\tau}_a$ (Max. flight time), $\delta_a, \delta_g, \delta_o$ (Robustness margins)
}

\Output{
    $(X, \Theta)$, $X {=} \{X^{\mu}\}_{\mu=1}^{m}$ (Mission plan), $\Theta {=} \{\Theta^{\mu}\}_{\mu=1}^{m}$ (Timing)
}

\caption{\scriptsize Precedence-Aware Robust UAV-UGV Planning}
\label{alg:precept}
\textbf{Initialize:} $\forall \mu$: $X^{\mu} \gets \emptyset$;\ \ $\Theta_a^\mu \leftarrow \mathbf{0}_{n \times (n+1)}$;\ \ $\Theta_g^\mu \leftarrow \mathbf{0}_n$;\ \ $\Theta_{ga}^\mu \leftarrow \mathbf{0}_{n-1}$;\ \ $\theta_s^\mu \gets 0$;\ \ $\theta_f^\mu \gets 0$;\ \ $i_{\mu} \gets 0$;\ \ $j_{\mu} \gets 2$;\ \ $\Gamma_{\mu} \gets 0$;\ \ $\eta_{\mu} \gets 0$;\ \ $C_{\mu} \gets \emptyset$;\ \ $T_{\mu} \gets 0$\;
$\mathcal{A} \gets \emptyset$;\ \ $\mathcal{R} \gets \{q \in \mathcal{P}_{\text{UAV}} \mid (p,q) \notin E_\prec \ \, \forall p \in \mathcal{P}_{\text{UAV}}\}$\;
\While{$\mathcal{R} \neq \emptyset$}{
    \ForEach{$\mu \in \{1,\ldots,m\}$ \textbf{\emph{and}} $q \in \mathcal{R}$}{
        $w_q \gets \max\big(\{0\} \cup \{t_p {+} \delta_o \mid (p,q) \in E_\prec\}\big)$\;
        $C' \gets \emptyset$;\ \ $r \gets \text{null}$;\ \ $c \gets \text{null}$;\ \ $\Gamma' \gets \text{null}$\;
        \If{$i_{\mu} \geq 1$}{
            $e \gets \Gamma_{\mu} + \eta_{\mu} + \underline{\theta}_a(X^{\mu}_{i_\mu,j_\mu-1}, q)$\;
            $\sigma_a \gets \max\big(0,\, w_q - e\big)$;\ \ $\sigma_g \gets 0$\;
            $P \gets \{\pi(X^{\mu}_{i_\mu,k})\}_{2 \leq k < j_{\mu}} \cup \{\pi(q)\}$\;
            $C' \gets \big\{ \hat{c} \in P \ \big|\ \eta_{\mu} {+} \underline{\theta}_a(X^{\mu}_{i_\mu,j_\mu-1},q) {+} \sigma_a {+} \underline{\theta}_a(q,\hat{c}) {+} \delta_a \leq \overline{\tau}_a \ \textbf{and}\ \underline{\theta}_g(X^{\mu}_{i_\mu,1},\hat{c}) {+} \delta_g \leq \overline{\tau}_a \big\}$\;
            $t \gets e + \sigma_a$\;
        }
        \If{$C' = \emptyset$}{
            $r \gets \pi(q)$\;
            \uIf{$i_{\mu} = 0$}{
                $e \gets \underline{\theta}_g(p^{\mu}_{\text{s}}, r) + \underline{\theta}_a(r, q)$\;
                $\sigma_g \gets \max\big(0,\, w_q - e\big)$;\ \ $\sigma_a \gets 0$\;
                $\Gamma' \gets \underline{\theta}_g(p^{\mu}_{\text{s}}, r) + \sigma_g$\;
            }
            \Else{
                $c \gets \mathop{\mathrm{arg\,min}}\limits_{\hat{c} \in C_{\mu}} \Big\{ \hat{b} + \max\big(\underline{\theta}_g(\hat{c},r),\, \tau_c(\hat{b})\big) \ \Big|\ \hat{b} = \max\big(\eta_{\mu} {+} \underline{\theta}_a(X^{\mu}_{i_\mu,j_\mu-1},\hat{c}),\, \underline{\theta}_g(X^{\mu}_{i_\mu,1},\hat{c})\big) \Big\}$\;
                $b \gets \max\big(\eta_{\mu} {+} \underline{\theta}_a(X^{\mu}_{i_\mu,j_\mu-1},c),\, \underline{\theta}_g(X^{\mu}_{i_\mu,1},c)\big)$\;
                $e \gets \Gamma_{\mu} + b + \max\big(\underline{\theta}_g(c,r),\, \tau_c(b)\big) + \underline{\theta}_a(r, q)$\;
                $\sigma_g \gets \max\big(0,\, w_q - e\big)$;\ \ $\sigma_a \gets 0$\;
                $\Gamma' \gets \Gamma_{\mu} + b + \max\big(\underline{\theta}_g(c,r),\, \tau_c(b)\big) + \sigma_g$\;
            }
            $t \gets \Gamma' + \underline{\theta}_a(r, q)$;\ \ $C' \gets \{r\}$\;
        }
        $\langle t, r, c, \sigma_a, \sigma_g, C', \Gamma' \rangle^{\mu,q} \gets \langle t, r, c, \sigma_a, \sigma_g, C', \Gamma' \rangle$\;
    }
    $(\mu^{\star}, q^{\star}) \gets \mathop{\mathrm{arg\,min}}\limits_{\mu \in \{1,\ldots,m\},\, q \in \mathcal{R}} \max\big(t^{\mu,q},\, \max_{\nu \neq \mu} T_{\nu}\big)$\;
    $\langle t, r, c, \sigma_a, \sigma_g, C', \Gamma' \rangle \gets \langle t, r, c, \sigma_a, \sigma_g, C', \Gamma' \rangle^{\mu^{\star}, q^{\star}}$\;
    $\mu \gets \mu^{\star}$;\ \ $q \gets q^{\star}$\;
    \uIf{$r = \text{null}$}{
        $X^{\mu}_{i_\mu,j_\mu} \gets q$\;
        $\theta^{\mu,i_\mu,j_\mu-1}_a \gets \underline{\theta}_a(X^{\mu}_{i_\mu,j_\mu-1}, q) + \sigma_a$\;
        $\eta_{\mu} \gets \eta_{\mu} + \theta^{\mu,i_\mu,j_\mu-1}_a$\;
    }
    \Else{
        \uIf{$i_{\mu} \geq 1$}{
            $X^{\mu}_{i_\mu,\,j_\mu:n+1} \gets X^{\mu}_{i_\mu,j_\mu-1}$;\ \ $X^{\mu}_{i_\mu,n+2} \gets c$\;
            $\theta^{\mu,i_\mu,n+1}_a \gets \underline{\theta}_a(X^{\mu}_{i_\mu,j_\mu-1}, c)$\;
            $\theta^{\mu,i_\mu}_g \gets \underline{\theta}_g(X^{\mu}_{i_\mu,1}, c)$\;
            $b \gets \max(\eta_{\mu} {+} \theta^{\mu,i_\mu,n+1}_a,\, \theta^{\mu,i_\mu}_g)$;\ \ $\theta^{\mu,i_\mu}_{ga} \gets \max\big(\underline{\theta}_g(c,r),\, \tau_c(b)\big) + \sigma_g$\;
        }
        \Else{
            $\theta^{\mu}_{s} \gets \underline{\theta}_g(p^{\mu}_{\text{s}}, r) + \sigma_g$\;
        }
        $i_{\mu} \gets i_{\mu}+1$;\ \ $j_{\mu} \gets 2$;\ \ $\Gamma_{\mu} \gets \Gamma'$\;
        $X^{\mu}_{i_\mu,1} \gets r$;\ \ $X^{\mu}_{i_\mu,2} \gets q$\;
        $\theta^{\mu,i_\mu,1}_a \gets \underline{\theta}_a(r, q)$;\ \ $\eta_{\mu} \gets \theta^{\mu,i_\mu,1}_a$\;
    }
    $C_{\mu} \gets C'$;\ \ $j_{\mu} \gets j_{\mu}+1$;\ \ $t_q \gets t$;\ \ $T_{\mu} \gets t$\;
    $\mathcal{A} \gets \mathcal{A} \cup \{q\}$\;
    $\mathcal{R} \gets \big(\mathcal{R} \setminus \{q\}\big) \cup \{s \mid (q,s) \in E_\prec,\ \, p \in \mathcal{A} \ \forall (p,s) \in E_\prec\}$\;
}

\color{black}
\For{\(\mu \gets 1\) \KwTo \(m\)}{
    \eIf{\(i_\mu \ge 1\)}{
        \(c \gets \mathop{\mathrm{arg\,min}}\limits_{\hat{c} \in C_{\mu}} \Big\{ \hat{b} + \underline{\theta}_g(\hat{c}, p^{\mu}_{\text{f}}) \ \Big|\ \hat{b} = \max\big(\eta_{\mu} {+} \underline{\theta}_a(X^{\mu}_{i_\mu,j_\mu-1},\hat{c}),\, \underline{\theta}_g(X^{\mu}_{i_\mu,1},\hat{c})\big) \Big\}\)\;
        \(X^{\mu}_{i_\mu,\,j_\mu:n+1} \gets X^{\mu}_{i_\mu,j_\mu-1}\);\ \ \(X^{\mu}_{i_\mu,n+2} \gets c\)\;
        \(\theta^{\mu,i_\mu,n+1}_a \gets \underline{\theta}_a(X^{\mu}_{i_\mu,j_\mu-1}, c)\)\;
        \(\theta^{\mu,i_\mu}_g \gets \underline{\theta}_g(X^{\mu}_{i_\mu,1}, c)\);\ \ \(\theta^{\mu}_{f} \gets \underline{\theta}_g(c, p^{\mu}_{\text{f}})\)\;
        \(X^{\mu}_{i_\mu+1:n,\,:} \gets c\)\;
    }{
        \(X^{\mu}_{1:n,\,:} \gets p^{\mu}_{\text{f}}\)\;
        \(\theta^{\mu}_s \gets \underline{\theta}_g(p^{\mu}_{\text{s}}, p^{\mu}_{\text{f}})\)\;
    }
}
\Return{$(X, \Theta)$}
\end{algorithm}
\color{black}
Alg.~\ref{alg:precept} operates iteratively, placing one monitoring point per loop (lines 3-47). Line 1 initializes the state for each team $\mu$: the plan matrix $X^\mu$ and timing matrices $\Theta_a^\mu, \Theta_g^\mu, \Theta_{ga}^\mu$ alongside their start/finish scalars $\theta_s^\mu, \theta_f^\mu$ strictly to zero, the active tour (row) index $i_\mu$, the column index $j_\mu$, the clock $\Gamma_\mu$ tracking the time the current tour is released from the UGV, the accumulated flight time $\eta_\mu$ of the tour, the candidate collect points $C_\mu$ of the tour, and the team's accumulated cost $T_\mu$. Line 2 initializes the placed set $\mathcal{A}$ and the ready set $\mathcal{R}$, holding the points whose required predecessors in $G_\prec$ are in $\mathcal{A}$.

Lines 4-26 evaluate every pairing of a team $\mu$ and a ready point $q$. Line 5 computes $w_q$, the earliest visit time permitted by the precedence constraints, ensuring $q$ occurs at least $\delta_o$ after all its placed predecessors. Line 6 initializes local provisional variables: candidate collect set $C'$, next release point $r$, collect point $c$, and tour release clock $\Gamma'$. 

Line 7 checks if team $\mu$ currently has an open tour ($i_\mu \ge 1$). If so, line 8 calculates the express time $e$, the earliest possible arrival at $q$, by adding the accumulated flight $\eta_\mu$ and the new flight leg to the release clock $\Gamma_\mu$. Line 9 calculates the airborne wait $\sigma_a$ necessary to delay this arrival until $w_q$, setting the ground wait $\sigma_g$ to 0. Line 10 gets the potential collect points $P$, the ground projections of all air points currently in the tour plus $q$. Line 11 filters $P$ to retain only the subset $C'$ that maintains UAV and UGV energy feasibility when $\sigma_a$ is included. Line 12 sets the final arrival time $t$ at $q$ by adding the airborne wait $\sigma_a$ to the express time $e$.

Line 13 checks if $C'$ is empty, which occurs if $q$ violates the energy check (line 11) or if it is the team's first tour. If empty, $q$ opens a new tour, and line 14 sets the new release point $r$ to $q$'s feasible ground projection $\pi(q)$. If this is the team's first tour (line 15), line 16 calculates the express time $e$ by summing the fastest UGV drive from $p_s^\mu$ to $r$ and the fastest UAV flight from $r$ to $q$. Line 17 calculates the required ground wait $\sigma_g$, setting $\sigma_a$ to 0. Line 18 sets the provisional release clock $\Gamma'$ by adding the drive time and the wait $\sigma_g$. If a prior tour is already open (line 19), it must be closed first. Line 20 selects a collect point $c \in C_\mu$ that minimizes the total time to finish the closing tour, recharge the UAV, and drive the UGV to the new release point $r$. Line 21 defines the closing tour's airborne time $b$ as the longer of the UGV's drive to $c$ and the UAV's total flight (the accumulated flight $\eta_\mu$ plus the landing leg from the latest point to $c$). Line 22 gets the express time $e$ at $q$ from the release clock $\Gamma_\mu$, adding the airborne time $b$, the gap time (the longer of the UGV transit to $r$ or the UAV recharge), and the initial flight to $q$. Line 23 calculates the required ground wait $\sigma_g$. Line 24 calculates the provisional release clock $\Gamma'$ for the new tour by adding the flight, gap time, and wait $\sigma_g$ to the release clock $\Gamma_\mu$. Line 25 sets the final arrival time $t$ at $q$ and sets the new tour's candidate collect set to $\{r\}$. Line 26 stores this evaluated option.

Line 27 selects the pairing \((\mu^\star, q^\star)\) minimizing the maximum accumulated team cost \(T_\nu\) (breaking ties lexicographically, if applicable). This distributes the workload rather than greedily visiting individual points early. Lines 28 and 29 unpack and commit this winning option.

Lines 30-47 commit the selected pairing. The algorithm prioritizes the standard case: line 30 checks if $r = \text{null}$, meaning $q$ simply joins the existing open tour. If so, line 31 writes $q$ into the next free matrix entry, and lines 32-33 apply the airborne wait $\sigma_a$ to the inbound flight leg and update the accumulated flight $\eta_\mu$. Conversely, if $r \neq \text{null}$ (line 34), a new tour is opened. Line 35 checks if a currently open tour must be closed first. If so, line 36 pads the unused air columns with the current location $X^{\mu}_{i_\mu,j_\mu-1}$, and lines 37-39 write the landing leg, UGV relocation, and transport timing, explicitly including the ground wait $\sigma_g$ in the transport time. For a team's first tour, line 41 instead writes the initial UGV drive time plus the wait $\sigma_g$ to $\theta_s^\mu$. Lines 42-44 then advance the team state: the row index $i_\mu$ increments, the column index $j_\mu$ resets to 2, the active release clock $\Gamma_\mu$ updates to the provisional clock $\Gamma'$, the accumulated flight $\eta_\mu$ is reset, and the new tour is initialized with $r$ and $q$ in the first two entries of the new row. Line 45 then finalizes the shared state updates: it commits the candidate collect set $C_\mu$, increments the column index $j_\mu$, sets the final visit time $t_q$, and updates the team's cost $T_\mu$. Lines 46-47 add $q$ to $\mathcal{A}$ and update $\mathcal{R}$ by admitting any successors whose predecessors are now all placed.

\color{black}
Lines 48-58 execute a closure pass once \(\mathcal{R} = \emptyset\). For active teams (\(i_\mu \ge 1\)), line 50 selects the final collect point \(c \in C_\mu\) minimizing the drive to \(p_f^\mu\). Line 51 closes the final active row, and lines 52-53 assign the landing leg, final UGV relocation, and final transport timing \(\theta_f^\mu\). If a team remains idle, lines 56-57 bypass these steps, routing the UGV directly from \(p_s^\mu\) to \(p_f^\mu\). Lines 54 and 56 pad unused matrix entries, and line 58 returns the plan \(X\) and timing \(\Theta\).

\color{black}

\subsection{Formal Guarantees}
\label{sec:guarantees}

We now state formal guarantees on the feasibility of the solution returned by
Alg.~\ref{alg:precept} and on its worst-case complexity.

\begin{assumption}Every point \(q \in \mathcal{P}_{\text{UAV}}\) admits a feasible single-point tour from its ground projection \(\pi(q)\) i.e., \(\underline{\theta}_a(\pi(q), q) + \underline{\theta}_a(q, \pi(q)) + \delta_a \le \overline{\tau}_a\) and \(\underline{\theta}_g(\pi(q), \pi(q)) + \delta_g \le \overline{\tau}_a\).\footnote{This is a mild condition: the UAVs must at least have the energy to complete a single-point tour (taking off from a stationary UGV, visiting the point, and landing back on the UGV).}
\label{assump1}
\end{assumption}

\begin{theorem}
\label{thm:main}
Under Assumption~\ref{assump1}, Alg.~\ref{alg:precept} returns a feasible solution to \eqref{opt_problem} if one
exists, and has worst-case time complexity $\mathcal{O}(mn^3)$, where $m$ is the number
of UAV-UGV teams and $n {=} |\mathcal{P}_{\text{UAV}}|$ is the number of monitoring points.
\end{theorem}

\begin{proof}\textit{(Feasibility.)} Alg.~\ref{alg:precept} iteratively constructs the mission plan via a while loop (lines 3-47). Since $G_\prec$ is a DAG, the shrinking subset of unplaced points always contains at least one point with no unplaced predecessors; by construction (lines 2, 47), this point belongs to the ready set $\mathcal{R}$. Each iteration assigns one point $q \in \mathcal{R}$ to a team. The algorithm appends $q$ to an open tour if energy constraints permit (lines 11, 30-33), or otherwise opens a new tour at the ground projection $\pi(q)$ (lines 13-14, 34-44). Under Assumption ~\ref{assump1}, this new tour is energy-feasible. Since a valid assignment is always available, the loop terminates in exactly $n$ iterations, placing every monitoring point into exactly one team's plan, satisfying \eqref{eq:cov}.
\color{black}

If $q$ joins an open tour, the algorithm enforces energy-feasibility (line 11) before committing the point to the plan (lines 30-33). If $q$ instead opens a new tour, the release and candidate collect points are set to $\pi(q)$ (lines 14, 25) and committed to the plan (lines 34-44). Under Assumption \ref{assump1}, this tour is also energy-feasible, satisfying \eqref{eq:air} and \eqref{eq:ground}.

If appending to an open tour, candidate collect points $P$ are generated using the feasible ground projection $\pi(\cdot)$ (line 10) and filtered into a subset $C'$ (line 11). If opening a new tour instead, the release point is set to the ground projection $\pi(q)$ (line 14) and $C'$ is populated with this projected point (line 25). In both cases, $C'$ updates the team's candidate set $C_\mu$ (line 45). Since the collect point $c$ is selected from $C_\mu$ (lines 20, 50), every release and collect point of the plan (lines 36, 43, 51, 54, 56) is from $\mathcal{Q}_f^g$, satisfying \eqref{eq:ground_pts}. 

The target $q \in \mathcal{P}_{\text{UAV}}$ is added to the tour either when appending to an open tour (line 31) or when opening a new tour (line 43). Unused entries within an active row are padded with the last visited point $X^{\mu}_{i_\mu,j_\mu-1}$ (lines 36, 51), while any completely unused subsequent rows are padded with the final collect point $c$ or $p_f^\mu$ (lines 54, 56), satisfying \eqref{eq:struct}.

The algorithm computes wait times $\sigma_a, \sigma_g \ge 0$ (lines 9, 17, 23). For active legs, each timing variable in $\Theta$ is initialized to zero (line 1) and assigned either the fastest travel time, or the fastest travel time plus the corresponding wait (lines 32, 37-39, 41, 44, 52-53, 57). Since the matrices are zero-initialized and $\sigma_a, \sigma_g \ge 0$, every active duration is greater than or equal to its fastest travel time, and padded edges remain zero, satisfying \eqref{eq:air_floor}, \eqref{eq:ground_floor}, \eqref{eq:transport_floor}, and \eqref{eq:boundary_floor}.

A point $q$ enters the ready set $\mathcal{R}$ only after all of its predecessors are placed (line 47). Since every placed predecessor already has a known visit time $t_p$ from a previous iteration (line 45), the algorithm can compute a lower bound $w_q$ that strictly enforces the ordering margin $\delta_o$ against all of them (line 5). By calculating and imposing the wait times $\sigma_a, \sigma_g$ (lines 9, 17, 23), the resulting candidate visit time $t$ is forced to satisfy $t \ge w_q$ (lines 12, 25). When $q$ is assigned to a team, its final visit time $t_q$ is committed as $t$ (line 45). Thus, $t_q \ge t_p + \delta_o$ for all $(p,q) \in E_\prec$, satisfying \eqref{eq:prec}.

Therefore, under Assumption \ref{assump1}, Algorithm \ref{alg:precept} returns a feasible solution to \eqref{opt_problem}, if one exists. 

\emph{(Complexity.)} Alg. 1 initializes the ready set $\mathcal{R}$ (line 2), taking $\mathcal{O}(n + \vert E_{\prec} \vert)$ time to scan $G_{\prec}$.

The main \texttt{while} loop (line 3) terminates in $n$ iterations. Within each iteration, a nested \texttt{foreach} loop (lines 4-26) evaluates all pairings of team $\mu \in \{1,\dots,m\}$ and $q \in \mathcal{R}$. Since $\vert \mathcal{R} \vert \le n$, this inner loop checks at most $mn$ options.

During this \texttt{foreach}, computing $w_q$ (line 5) across all $q \in \mathcal{R}$ for a given team takes $\mathcal{O}(\vert E_{\prec} \vert)$ time. Checking energy feasibility (lines 10-11) requires finding feasible ground projections $\pi(\cdot)$ for each point in the active tour of at most $n$ points. Assuming a bounded number of obstacle faces, each projection takes fixed time, requiring $\mathcal{O}(n)$ time per pairing. Selecting a collect point $c \in C_\mu$ (line 20) scans at most $n$ candidates, adding $\mathcal{O}(n)$ time per pairing. All remaining calculations and assignments within the loop take fixed time. Consequently, the inner loop operations require a total of $\mathcal{O}(m\vert E_{\prec} \vert + mn^2)$ time per outer iteration.

After the \texttt{foreach}, selecting the $(\mu^\star, q^\star)$ (line 27) takes $\mathcal{O}(mn)$ time per iteration. Committing this selection and updating the state matrices (lines 28-45) requires padding up to $n$ column entries (line 36), taking $\mathcal{O}(n)$ time per iteration.

Updating $\mathcal{R}$ (line 47) takes an aggregate $\mathcal{O}(n + \vert E_{\prec} \vert)$ time across all $n$ iterations by tracking in-degrees.

Finally, a \texttt{for} loop over the $m$ teams (lines 48-58) scans at most $n$ candidates to select the final collect point (line 50), taking $\mathcal{O}(n)$ time per team. The remaining matrix assignments (lines 51, 54, 56) pad up to $n$ rows of length $n+2$, taking $\mathcal{O}(n^2)$ time per team, requiring a total of $\mathcal{O}(mn^2)$ time.

Summing these yields \(\mathcal{O}(n + \vert E_{\prec} \vert + mn\vert E_{\prec} \vert + mn^3 + mn^2 + n^2 + mn^2)\). Since \(G_{\prec}\) is a directed acyclic graph, \(\vert E_{\prec} \vert \le n(n-1)/2\). Thus, the \(\mathcal{O}(mn^3)\) term dominates, yielding an overall worst-case time complexity of \(\mathcal{O}(mn^3)\).
\end{proof}
\color{black}

\section{Numerical Evaluations}
\label{sec:numerical}
We evaluate Alg. \ref{alg:precept} on: (i) scalability, (ii) precedence density effect, (iii) robustness margin effect, and compare against (iv) branch  and cut, (v) simulated annealing, and a (vi) decoupled precedence greedy baseline. Python 3.9 on running on Intel i9-13900H CPU with 16GB RAM was used.

The environment is \eqref{eq:env}, and $\mathcal{P}_{\text{UAV}}$ is drawn
uniformly from $\mathcal{Q}^a_f$ at altitude $\underline z$. A VTOL
UAV is used, flying the horizontal component $\ell_H$ of each leg at speed $s_H$ and the vertical component $\ell_V$ at speed
$s_V < s_H$ \cite{guo2020vertical}.\footnote{In practice, most UAVs have higher
horizontal speeds than vertical.} $G_\prec$ is generated by drawing a random permutation of
$\mathcal{P}_{\text{UAV}}$ and adding $\lfloor d n \rceil$ edges between random
pairs, each directed from the earlier point in the permutation to the later, so
no cycle can form.

Unless otherwise specified, (i) all results\footnote{For brevity, the makespan $\max_\mu \tau({X}^\mu, \Theta^\mu)$ is denoted by $\max_\mu \tau$.} are reported as mean $\pm$ standard deviation over \textit{25 randomly generated instances}, each with its own $G_\prec=(\mathcal{P}_{\text{UAV}}, E_\prec)$, and (ii) the parameters of Table~\ref{tab:params} are used.
\color{black}

\begin{table}[htbp]
\centering
\caption{Mission and Method Parameters}
\label{tab:params}
\begin{tabularx}{\columnwidth}{@{}l X@{}}
\hline
Parameter & Value \\
\hline
$\bar x, \bar y$ (Environment size) & 4000 m \\
$\underline z$ (Monitoring height) & 100 m \\
$\{p^{\mu}_{\text{s}}, p^{\mu}_{\text{f}}\}_{\mu=1}^{m}$ (Start, finish positions) & $(0,0,0)$, $(4000,4000,0)$ \\
$d$ (Precedence density) & $0.5$ \\
$\ell(u,v)$ (Shortest path) & $\lVert u {-} v \rVert_2$ \\
$s_H, s_V$ (UAV cruise, climb speed) & 10 m/s, 2 m/s \\
$s_g$ (UGV speed) & 2.5 m/s \\
$\underline{\theta}_a(u,v)$ (UAV fastest travel time) & $\ell_H(u,v)/s_H + \ell_V(u,v)/s_V$ \\
$\underline{\theta}_g(u,v)$ (UGV fastest travel time) & $\ell(u,v)/s_g$ \\
$\tau_c(\cdot)$ (Recharge function) & $\gamma \max\big(\tau_a(X^\mu_i, \Theta^\mu),\, \tau_g(X^\mu_i, \Theta^\mu)\big)$ \\
$\gamma$ (Recharge ratio) & 1 \\
$\overline{\tau}_a$ (Max flight time) & 600 s \\
$\delta_a, \delta_g$ (Energy robustness margins) & 0 s \\
$\delta_o$ (Ordering robustness margin) & 0 s \\
\hline
\end{tabularx}
\end{table}

\subsection{Scalability}
\label{subsec:scalability}

Table~\ref{tab:scalability} evaluates the scalability of Alg.~\ref{alg:precept} across \(n \in \{25, 50, 75, 100\}\) and \(m \in \{1, 2, 4, 10\}\). Empirical computation times scale at \(n^{2.15}\)--\(n^{2.26}\) and \(m^{0.62}\)--\(m^{0.68}\), consistent with the theoretical \(\mathcal{O}(mn^3)\) bound (Thm.~\ref{thm:main}). Makespan predictably increases with \(n\) and decreases as \(m\) parallelizes the teams. Alg. \ref{alg:precept} reliably computes large-scale instances, solving the most demanding configuration (\(n {=} 100\), \(m {=} 10\)) in under $0.35$\,s.

\begin{table}[htpb]
\centering
\caption{Makespan (\(\max_\mu \tau\)) and computation time under Alg.~\ref{alg:precept} across numbers of points ($n$) and numbers of teams ($m$)}
\label{tab:scalability}
\setlength{\tabcolsep}{3pt}
\footnotesize
\begin{tabular}{c rrrr}
\toprule
 & \multicolumn{2}{c}{\(m=1\)} & \multicolumn{2}{c}{\(m=2\)} \\
\cmidrule(lr){2-3}\cmidrule(lr){4-5}
\(n\) & \(\max_\mu \tau\) (s) & Comp. (s) & \(\max_\mu \tau\) (s) & Comp. (s) \\
\midrule
 25 &  8900 \(\pm\) 1200 & 0.004 \(\pm\) 0.001 & 5400 \(\pm\) 850 & 0.006 \(\pm\) 0.001 \\
 50 & 10900 \(\pm\) 1700 & 0.015 \(\pm\) 0.002 & 6800 \(\pm\) 810 & 0.026 \(\pm\) 0.003 \\
 75 & 13300 \(\pm\) 1400 & 0.038 \(\pm\) 0.004 & 7500 \(\pm\) 760 & 0.064 \(\pm\) 0.006 \\
100 & 15100 \(\pm\) 1600 & 0.071 \(\pm\) 0.005 & 8300 \(\pm\) 880 & 0.126 \(\pm\) 0.015 \\
\midrule
 & \multicolumn{2}{c}{\(m=4\)} & \multicolumn{2}{c}{\(m=10\)} \\
\cmidrule(lr){2-3}\cmidrule(lr){4-5}
\(n\) & \(\max_\mu \tau\) (s) & Comp. (s) & \(\max_\mu \tau\) (s) & Comp. (s) \\
\midrule
 25 & 4000 \(\pm\) 570 & 0.009 \(\pm\) 0.001 & 3500 \(\pm\) 290 & 0.015 \(\pm\) 0.002 \\
 50 & 4400 \(\pm\) 470 & 0.042 \(\pm\) 0.004 & 3500 \(\pm\) 320 & 0.072 \(\pm\) 0.007 \\
 75 & 4900 \(\pm\) 480 & 0.107 \(\pm\) 0.011 & 3600 \(\pm\) 230 & 0.181 \(\pm\) 0.016 \\
100 & 5300 \(\pm\) 540 & 0.205 \(\pm\) 0.021 & 3900 \(\pm\) 310 & 0.342 \(\pm\) 0.033 \\
\bottomrule
\end{tabular}
\end{table}

\subsection{Effect of Precedence Density}
\label{subsec:precedence_eval}
We evaluate precedence densities $d \in \{0.5, 1, 3, 8, 20\}$ for $m \in \{1, 4, 10\}$ and $n {=} 100$, spanning $50$ to $2000$ precedence edges ($|E_\prec|$), summarized in Table~\ref{tab:precedence}. Increasing $d$ from $0.5$ to $20$ severely restricts routing freedom; the mean size of the ready set $\mathcal{R}$ drops from $31.6$ to $2.0$, and the longest chain ($c$) forces $51\%$ of the points into a sequential bottleneck ($c/n {=} 0.51$). To counter this, Alg.~\ref{alg:precept} parallelizes teams. Even when dense edges (high $d$) merge the entire $G_\prec$ into a single connected component, the algorithm continues to distribute points across teams. Consequently, expanding from $m {=} 1$ to $m {=} 10$ teams achieves an $89\%$ makespan reduction at $d {=} 20$, exceeding the $74\%$ reduction observed in $d {=} 0.5$.

\color{blue}
\begin{table}[htpb]
\centering
\caption{Makespan (\(\max_\mu \tau\)) of Alg.~\ref{alg:precept} across precedence densities (\(d\)) and numbers of teams (\(m\)) at \(n{=}100\). \(c\) is the longest chain in \(G_\prec\), and Red.\ is the makespan reduction from \(m{=}1\) to \(m{=}10\)}
\label{tab:precedence}
\setlength{\tabcolsep}{3pt}
\footnotesize
\begin{tabular*}{\columnwidth}{@{\extracolsep{\fill}}cc ccc c@{}}
\toprule
 & & \multicolumn{3}{c}{\(\max_\mu \tau\) (s)} & \\
\cmidrule(lr){3-5}
\(d\) & \(c/n\) & \(m{=}1\) & \(m{=}4\) & \(m{=}10\) & Red.\ (\%) \\
\midrule
 0.5 & 0.04 & 15100 \(\pm\) 1600 &  5300 \(\pm\)  540 & 3900 \(\pm\) 310 & 74 \\
 1   & 0.06 & 17900 \(\pm\) 1600 &  6200 \(\pm\)  570 & 4200 \(\pm\) 260 & 77 \\
 3   & 0.12 & 27600 \(\pm\) 2500 &  7600 \(\pm\)  530 & 4900 \(\pm\) 240 & 82 \\
 8   & 0.26 & 40400 \(\pm\) 4600 &  9900 \(\pm\)  890 & 5500 \(\pm\) 360 & 86 \\
 20  & 0.51 & 57100 \(\pm\) 4800 & 12000 \(\pm\) 1100 & 6000 \(\pm\) 410 & 89 \\
\bottomrule
\end{tabular*}
\end{table}

\color{black}

\subsection{Effect of Robustness Margins}
\label{subsec:robustness_eval}
We evaluate the energy margins ($\delta_a {=} \delta_g \in [0, 490]$~s) and the ordering margin ($\delta_o \in [0, 1600]$~s) independently for $m \in \{1, 2, 4, 10\}$ teams at $n {=} 100$, summarized in Fig.~\ref{fig:robustness}. These margins affect parallelism differently. Increasing energy margins shortens tours, but these tours remain easy to distribute; thus, the makespan reduction from $m {=} 1$ to $m {=} 10$ remains stable around $75\%$ to $80\%$. Conversely, increasing the ordering margin forces teams to wait, causing the makespan curves to converge. At large $\delta_o$, the makespan reduction from $m {=} 1$ to $m {=} 10$ drops to $58\%$, and the benefit of expanding from $m {=} 4$ to $m {=} 10$ shrinks from $28\%$ to $8\%$. Small ordering margins ($\delta_o \le 200$\,s) are absorbed with almost no extra time. Ultimately, the margins constrain the problem differently: $\delta_a$ and $\delta_g$ impose spatial costs that multiple teams can divide, whereas $\delta_o$ imposes strict wait times that inherently limit parallel efficiency.
\color{black}

\begin{figure}[htpb]
    \centering
    \includegraphics[width=1.0\linewidth]{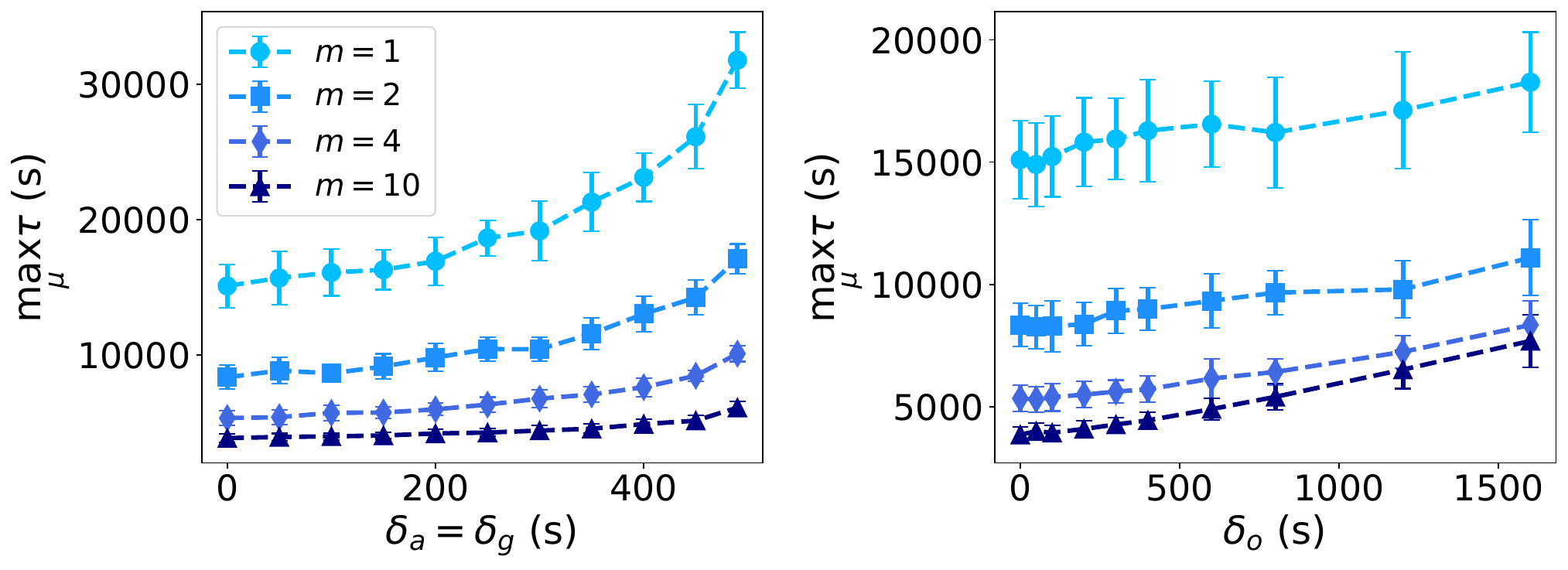}
    \caption{Makespan under Alg.~\ref{alg:precept} against the energy margins
    $\delta_a {=} \delta_g$ (left) and the ordering margin $\delta_o$ (right),
    at $n {=} 100$, each sweep holding the remaining margins at
    zero. Error bars are $\pm 1$ std.}
    \label{fig:robustness}
\end{figure}

\subsection{Comparison to Branch and Cut}
\label{subsec:bc}

We encode \eqref{opt_problem} as a Mixed-Integer Quadratically Constrained Program (MIQCP) solved with Branch \& Cut (B\&C) via Gurobi\footnote{MipGap: 5\%, TimeLimit: 600\,s, FeasibilityTol: \(10^{-4}\), IntFeasTol: \(10^{-5}\), OptimalityTol: \(10^{-3}\), NonConvex: 2} to establish a near-optimal reference (Table~\ref{tab:bc_comparison}). As expected, B\&C experiences combinatorial explosion, therefore, we use $n \in \{2,3,4\}$ and $m\in\{1,2\}$ for tractability. As \(n\) increases from 2 to 4, computation time scales exponentially, surging to over 200 seconds for \(m=2\) at \(n=4\) to reach the 5\% optimality threshold. In contrast, Alg.~\ref{alg:precept} computes in less than a millisecond across all configurations, achieving makespans within roughly \(8\%\) to \(28\%\) of near-optimal.

\begin{table}[htpb]
\centering
\caption{Makespan (\(\max_\mu \tau\)), computation time, and optimality gap of the Branch and Cut (B\&C) against Alg.~\ref{alg:precept} across different number of points \((n)\) and team sizes (\(m\))}
\label{tab:bc_comparison}
\setlength{\tabcolsep}{3pt}
\footnotesize
\resizebox{\columnwidth}{!}{%
\begin{tabular}{cc rrr rr}
\toprule
 & & \multicolumn{3}{c}{Branch \& Cut} & \multicolumn{2}{c}{Alg.~\ref{alg:precept}} \\
\cmidrule(lr){3-5} \cmidrule(lr){6-7}
\(m\) & \(n\) & \(\max_\mu \tau\) (s) & Comp. (s) & Gap (\%) & \(\max_\mu \tau\) (s) & Comp. (s) \\
\midrule
1 & 2 & 2800 \(\pm\) 380 & 0.19 \(\pm\) 0.12 & 1.2 \(\pm\) 2.0 & 3100 \(\pm\) 750 & \(<0.001\) \\
  & 3 & 2800 \(\pm\) 360 & 1.4 \(\pm\) 1.0 & 1.5 \(\pm\) 1.9 & 3200 \(\pm\) 800 & \(<0.001\) \\
  & 4 & 3100 \(\pm\) 330 & 19 \(\pm\) 16 & 2.1 \(\pm\) 1.8 & 3900 \(\pm\) 890 & \(<0.001\) \\
\midrule
2 & 2 & 2600 \(\pm\) 290 & 0.49 \(\pm\) 0.33 & 1.5 \(\pm\) 1.8 & 2800 \(\pm\) 310 & \(<0.001\) \\
  & 3 & 2600 \(\pm\) 330 & 13 \(\pm\) 12 & 2.8 \(\pm\) 2.2 & 2800 \(\pm\) 210 & \(<0.001\) \\
  & 4 & 2500 \(\pm\) 360 & 210 \(\pm\) 110 & 1.5 \(\pm\) 1.6 & 3200 \(\pm\) 490 & \(<0.001\) \\
\bottomrule
\end{tabular}%
}
\end{table}

\subsection{Comparison to Simulated Annealing}
We compare Alg. \ref{alg:precept} against Simulated Annealing (SA) using a logarithmic cooling schedule \cite{nourani1998comparison} with \(T_{min}=0.1\) (Table \ref{tab:sa_comparison}). We restrict this evaluation to \(n=15\) and $m\in\{1, 2\}$ because SA computation times become prohibitive for larger \(n\) and \(m\). \textcolor{black}{SA relies on local search, which struggles to explore efficiently because standard local search operators frequently generate precedence-violating sequences, resulting in high rejection rates,} even at its best configuration (200k steps). By directly constructing feasible assignments, Alg. \ref{alg:precept} avoids this bottleneck, achieving shorter makespans (\(7200 \pm 1550\)\,s for \(m=1\), \(4500 \pm 790\)\,s for \(m=2\)) while reducing computation time by roughly four orders of magnitude $(0.001 \pm 0.000$ s for $m=1$, and $0.002 \pm 0.000$ s for $m=2$).

\begin{table}[htpb]
\centering
\caption{Makespan (\(\max_\mu \tau\)) and computation time of Simulated
Annealing across maximum temperatures (\(T_{\max}\)) and cooling steps for different number of teams (\(m\))}
\label{tab:sa_comparison}
\setlength{\tabcolsep}{3pt}
\footnotesize
\resizebox{\columnwidth}{!}{%
\begin{tabular}{cl rrr rrr}
\toprule
 & & \multicolumn{3}{c}{Comp. (s)} & \multicolumn{3}{c}{\(\max_\mu \tau\) (s)} \\
\cmidrule(lr){3-5} \cmidrule(lr){6-8}
\(m\) & \(T_{\max}\) & 100k & 150k & 200k & 100k & 150k & 200k \\
\midrule
1 &  500 & 18.6 \(\pm\) 1.2 & 27.7 \(\pm\) 1.8 & 36.8 \(\pm\) 2.3 & 8600 \(\pm\) 1200 & 8100 \(\pm\) 1200 & 7800 \(\pm\) 1200 \\
  &  750 & 18.9 \(\pm\) 1.2 & 28.3 \(\pm\) 2.0 & 37.7 \(\pm\) 2.5 & 8400 \(\pm\) 1200 & 8100 \(\pm\) 1200 & 7800 \(\pm\) 1200 \\
  & 1000 & 19.6 \(\pm\) 1.6 & 29.3 \(\pm\) 2.4 & 38.7 \(\pm\) 3.1 & 8800 \(\pm\) 1200 & 8200 \(\pm\) 1200 & 7900 \(\pm\) 1100 \\
\midrule
2 &  500 & 37.8 \(\pm\) 2.4 & 55.8 \(\pm\) 3.3 & 73.5 \(\pm\) 4.0 & 6400 \(\pm\) 1000 & 6000 \(\pm\) 900 & 5700 \(\pm\) 790 \\
  &  750 & 41.4 \(\pm\) 4.7 & 61.7 \(\pm\) 7.5 & 82.0 \(\pm\) 9.8 & 7200 \(\pm\) 1000 & 6800 \(\pm\) 1100 & 6400 \(\pm\) 1100 \\
  & 1000 & 46.3 \(\pm\) 5.3 & 69.2 \(\pm\) 8.7 & 92 \(\pm\) 12 & 7900 \(\pm\) 1200 & 7400 \(\pm\) 1200 & 7200 \(\pm\) 1200 \\
\bottomrule
\end{tabular}%
}
\end{table}

\color{black}

\color{black}
\subsection{Comparison to the Decoupled Precedence Baseline}
\label{subsec:greedy}

To isolate the impact of cross-team precedence, we compare Alg.~\ref{alg:precept} against a decoupled baseline that assigns each weakly connected component of \(G_\prec\) to an individual team. The baseline estimates workloads via a precedence-respecting nearest-neighbor walk, distributes components using a Longest Processing Time (LPT) heuristic (i.e., assigning the component with the largest travel time to the team with the lowest accumulated time at the moment), and applies the identical tour-cutting and timing subroutines as Alg.~\ref{alg:precept}. Keeping precedence edges internal guarantees feasibility and speeds up computation by avoiding cross-team timing checks. Evaluated for \(n=100\), $m\in\{1,4,10\}$, and $d\in\{0.5, 3, 8, 20\}$ (Table~\ref{tab:baseline_comparison}), the baseline and Alg.~\ref{alg:precept} yield similar makespans for a single team. However, at \(d \ge 8.0\), \(G_{\prec}\) merges into a single connected component. The baseline assigns the entire mission to one team, freezing the makespan at $58800$~s (for \(d=20\)) regardless of \(m\). Alg. \ref{alg:precept} successfully distributes ordered points across multiple teams, achieving an \(89.8\%\) makespan reduction for \(m=10\) at \(d=20\): a parallelization trend that continues to scale for larger number of teams.

By skipping cross-team comparisons, the baseline computes faster in sparse scenarios (up to 57 times faster at \(d=0.5\)). However, this trade-off heavily favors Alg.~\ref{alg:precept} in high-density missions, where it yields a nearly 10 times shorter makespan for a negligible time cost ($0.043$~s versus $0.010$~s at \(d=20, m=10\)). Furthermore, the computation time of Alg.~\ref{alg:precept} inversely scales with precedence density while operating strictly within its \(O(mn^3)\) worst-case guarantee. Denser precedence constraints shrink the ready set \(\mathcal{R}\), which directly reduces the \(m \times |\mathcal{R}|\) pairings evaluated per iteration. Conversely, sparse precedence increases \(|\mathcal{R}|\), causing Alg.~\ref{alg:precept} to evaluate a higher number of cross-team pairings.

\begin{table}[htpb]
\centering
\caption{Makespan (\(\max_\mu \tau\)) and computation time comparison between Alg.~\ref{alg:precept} and the decoupled precedence baseline for different precedence densities (\(d\)) and numbers of teams (\(m\)), at \(n=100\)}
\label{tab:baseline_comparison}
\setlength{\tabcolsep}{3.5pt}
\footnotesize
\resizebox{\columnwidth}{!}{%
\begin{tabular}{cc rr rr}
\toprule
 & & \multicolumn{2}{c}{\(\max_\mu \tau\) (s)} & \multicolumn{2}{c}{Comp. (s)} \\
\cmidrule(lr){3-4} \cmidrule(lr){5-6}
\(d\) & \(m\) & \multicolumn{1}{c}{Baseline} & \multicolumn{1}{c}{Alg.~\ref{alg:precept}} & \multicolumn{1}{c}{Baseline} & \multicolumn{1}{c}{Alg.~\ref{alg:precept}} \\
\midrule
0.5 & 1  & \(15400 \pm 2100\) & \(\mathbf{15100 \pm 1600}\) & \(\mathbf{0.007 \pm 0.000}\) & \(0.071 \pm 0.005\) \\
    & 4  & \(10400 \pm 1100\) & \(\mathbf{5300 \pm 540}\)   & \(\mathbf{0.006 \pm 0.000}\) & \(0.205 \pm 0.021\) \\
    & 10 & \(8900 \pm 1700\)  & \(\mathbf{3900 \pm 310}\)   & \(\mathbf{0.006 \pm 0.001}\) & \(0.342 \pm 0.033\) \\
\midrule
3.0 & 1  & \(28800 \pm 2900\) & \(\mathbf{27600 \pm 2500}\) & \(\mathbf{0.007 \pm 0.001}\) & \(0.017 \pm 0.002\) \\
    & 4  & \(28700 \pm 2900\) & \(\mathbf{7600 \pm 530}\)   & \(\mathbf{0.008 \pm 0.000}\) & \(0.055 \pm 0.007\) \\
    & 10 & \(28700 \pm 2900\) & \(\mathbf{4900 \pm 240}\)   & \(\mathbf{0.007 \pm 0.001}\) & \(0.104 \pm 0.014\) \\
\midrule
8.0 & 1  & \(42400 \pm 3600\) & \(\mathbf{40400 \pm 4600}\) & \(\mathbf{0.008 \pm 0.000}\) & \(0.010 \pm 0.001\) \\
    & 4  & \(42400 \pm 3600\) & \(\mathbf{9900 \pm 890}\)   & \(\mathbf{0.008 \pm 0.001}\) & \(0.029 \pm 0.004\) \\
    & 10 & \(42400 \pm 3600\) & \(\mathbf{5500 \pm 360}\)   & \(\mathbf{0.008 \pm 0.001}\) & \(0.061 \pm 0.006\) \\
\midrule
20  & 1  & \(58800 \pm 5200\) & \(\mathbf{57100 \pm 4800}\) & \(0.009 \pm 0.000\) & \(\mathbf{0.007 \pm 0.001}\) \\
    & 4  & \(58800 \pm 5200\) & \(\mathbf{12000 \pm 1100}\) & \(\mathbf{0.010 \pm 0.001}\) & \(0.020 \pm 0.001\) \\
    & 10 & \(58800 \pm 5200\) & \(\mathbf{6000 \pm 410}\)   & \(\mathbf{0.010 \pm 0.001}\) & \(0.043 \pm 0.004\) \\
\bottomrule
\end{tabular}%
}
\end{table}
\color{black}

\section{Conclusion}
\color{black}
We addressed precedence-constrained planning for energy-constrained UAV-UGV teams. We formulated a mixed-integer program \eqref{opt_problem} for this NP-hard planning problem and introduced a polynomial-time heuristic (Alg.~\ref{alg:precept}) that simultaneously constructs assignments, energy-feasible tours, and travel timing. Guaranteed to find feasible solutions in \(\mathcal{O}(mn^3)\) time (Thm.~\ref{thm:main}), our approach effectively distributes ordered points across multiple teams. Numerical evaluations demonstrate that Alg.~\ref{alg:precept} scales to large instances that are intractable for exact solvers, and yields significantly shorter makespans than metaheuristic and heuristic baselines. Future work will extend the method to support additional spatio-temporal constraints, such as time windows and temporal logic. Additionally, we will incorporate online replanning to mitigate travel time disturbances during execution, alongside validation through real-world experimentation.

\color{black}
\bibliographystyle{IEEEtran}
\bibliography{references}   

\end{document}